\documentclass{article} 
\usepackage{iclr2026_conference,times}
\usepackage{algorithm}
\usepackage{algpseudocode}
\algtext*{EndFor}%
\algtext*{EndIf}%
\algtext*{EndWhile}%
\algtext*{EndProcedure}%
\usepackage{amsmath}
\usepackage{amsthm}
\usepackage{xcolor}
\usepackage{adjustbox}

\newcommand{\best}[1]{\textbf{#1}}
\newcommand{\second}[1]{\underline{#1}}

\usepackage{multirow}
\usepackage{makecell}
\usepackage{amssymb} 

\usepackage{amsmath,amsfonts,bm}

\def\eqref#1{equation~\ref{#1}}

\def\1{\bm{1}}

\DeclareMathAlphabet{\mathsfit}{\encodingdefault}{\sfdefault}{m}{sl}
\SetMathAlphabet{\mathsfit}{bold}{\encodingdefault}{\sfdefault}{bx}{n}

\usepackage{booktabs} 
\usepackage{graphicx} 
\usepackage{multirow}
\usepackage{colortbl} 
\usepackage{hyperref}
\hypersetup{hidelinks}
\usepackage{cleveref}
\usepackage{url}
\usepackage{xspace}
\usepackage{booktabs}
\usepackage{multirow}
\usepackage{graphicx}
\usepackage{booktabs}
\newtheorem{proposition}{Proposition}[section]
\newtheorem{corollary}[proposition]{Corollary}
\newcommand{\alg}{\texttt{SPOT}}

\title{SPOT: Sparse Probing and Outcome \\
Calibration for On-Policy Distillation}

\author{Zikun Qu$^{1}$,
Min Zhang$^{2}$\thanks{Corresponding authors}\hspace{1.6mm},
Mingze Kong$^{1}$,
Zhiwei Shang$^{1}$,
Zhengyu Chen$^{3}$
Yikun Ban$^{4}$,\\
\textbf{Shuang Qiu}$^{5}$,
\textbf{Zhongxiang Dai}$^{1*}$\\
$^{1}$The Chinese University of Hong Kong, Shenzhen,
$^{2}$East China Normal University,\\
$^{3}$Meituan LongCat
$^{4}$Beihang University,
$^{5}$City University of Hong Kong
}

\iclrfinalcopy 
\begin{document}

\maketitle


\begin{abstract}
On-policy distillation (OPD) provides dense teacher supervision on student-generated trajectories, but standard reverse-KL training can assign insufficient probability to other plausible continuations. Teacher entropy alone does not reveal whether uncertainty is concentrated among a few plausible next tokens or dispersed over a long probability tail, nor whether the student already represents those candidates well. Moreover, local teacher probabilities may not predict downstream success. We introduce \emph{\underline{S}parse \underline{P}robing and \underline{O}utcome-calibrated \underline{T}argets} OPD (\alg{}), which addresses two coupled decisions, \emph{where to probe} and \emph{what to distill}, through an acquisition--exploration--exploitation procedure. During acquisition, a position-level score combines normalized teacher entropy, the probability mass captured by a small top-$k$ candidate set, and student--teacher mismatch to allocate a limited probing budget. During exploration, \alg{} evaluates teacher-proposed candidates through verifier-scored student continuations. During exploitation, these outcomes produce a closed-form, KL-regularized target that favors candidates with better downstream outcomes while remaining anchored to the teacher distribution. Extensive experiments across multiple student models and reasoning benchmarks demonstrate the effectiveness of \alg{} in improving reasoning performance while balancing solution quality and coverage.
\end{abstract}

\vspace{-1mm}
\section{Introduction}
\label{sec:intro}
\vspace{-1mm}

Transferring the reasoning capabilities of large language models to smaller student models is an important goal in LLM post-training.
Supervised fine-tuning and off-policy distillation leverage expert- or teacher-generated trajectories, but train the student on contexts that differ from those induced by its own predictions, leading to exposure bias and compounding errors at inference
\citep{bengio2015scheduledsamplingsequenceprediction,ranzato2016sequenceleveltrainingrecurrent}.
On-policy reinforcement learning (RL) instead optimizes student-generated rollouts, but reasoning-oriented RL typically relies on sequence-level or terminal verifier rewards, providing limited fine-grained credit assignment
\citep{shao2024deepseekmathpushinglimitsmathematical,guo2025learningfocuscausalattention}.
On-policy distillation (OPD) combines these advantages: it trains on student-generated prefixes while retaining dense token-level teacher feedback
\citep{agarwal2024onpolicydistillationlanguagemodels,lu2025onpolicydistillation}.

Yet standard OPD typically minimizes reverse KL, whose mode-seeking behavior favors the teacher's dominant continuation but can assign insufficient probability to other plausible continuations, potentially limiting solution coverage.
EOPD addresses this limitation by using teacher entropy as a trigger: at high-entropy positions, it augments reverse-KL training with a top-$k$ approximation of forward KL to preserve plausible local alternatives~\citep{Minka2005DivergenceMA,jin2026entropyawareonpolicydistillationlanguage}.
However, high entropy alone does not reveal whether the teacher's uncertainty is concentrated among a few plausible next tokens or dispersed over a long probability tail.
Nor does it reveal whether the student already assigns sufficient probability to and similarly ranks those plausible tokens.
Evaluating an alternative next token requires rolling out a student continuation and checking its final outcome, making such probing costly.
An entropy threshold can flag uncertain positions, but by itself cannot prioritize where a limited probing budget will be most useful.
The first challenge is therefore to allocate that budget to positions where the teacher assigns substantial probability to a small set of alternatives that the student does not yet represent well.

Selecting a candidate position does not yet determine \emph{what to distill}.
On student-generated prefixes, the teacher's local next-token probabilities need not predict downstream success: a token assigned higher probability by the teacher may yield an unsuccessful continuation under the current student policy, whereas a lower-probability alternative may yield a successful one~\citep{li2026rethinkingonpolicydistillationlarge,hou2026uniopdunifyingonpolicydistillation}.
We therefore treat the teacher distribution as a proposal prior over candidate branches, rather than as definitive evidence of downstream success, and calibrate it using the outcomes of student continuations.
Thus, teacher uncertainty and student mismatch help prioritize where to acquire additional evidence, while verified downstream outcomes determine how that evidence should modify the supervision target.

These observations motivate \textbf{Sparse Probing and Outcome-calibrated Targets for on-policy distillation} (\alg), which reframes uncertainty-aware distillation around two coupled decisions: \emph{where to probe} and \emph{what to distill}.
\alg{} addresses them through a three-stage acquisition--exploration--exploitation procedure.
During acquisition, a lightweight position-level score $s_t$ prioritizes positions that satisfy three conditions: the teacher assigns meaningful probability to multiple next tokens, most of the teacher's probability mass lies within a small top-$k$ candidate set, and the student either underweights or differently ranks those candidates.
Because these factors are multiplied, a low value on any one condition lowers the position's overall probing priority.
The score therefore estimates \emph{where additional evidence may be useful}; it does not identify which candidates yield successful student continuations, which is assessed only after the candidates are rolled out and verified.
During exploration, at each selected position, \alg{} appends each candidate from the teacher's top-$k$ set in turn, rolls out a continuation under the student policy, and evaluates the completed continuation with a verifier.
During exploitation, the verified continuation values produce a closed-form, KL-regularized target: candidates with better downstream outcomes receive more probability, while the target remains anchored to the teacher distribution.
\alg{} applies this additional local loss only at positions where at least one tested candidate receives positive verifier reward.
This allows downstream evidence to modify the local token-level target rather than serving only as a trajectory-level score.

In summary, our contributions are threefold:
\vspace{-5pt}

{
\setlength{\leftmargini}{1.5em}
\begin{itemize}
    \item \textbf{A two-decision formulation.}
    We formulate selective supervision in OPD as determining \emph{where} additional outcome evidence is worth acquiring and \emph{how} that evidence should be converted into a supervision target, separating position selection from target construction.

    \item \textbf{Sparse probing and outcome-calibrated targets.}
    We propose \alg{}, which probes positions via teacher uncertainty, top-$k$ mass, and student mismatch, then derives a closed-form target by reward-tilting the teacher distribution with verifier-scored student continuations.

    \item \textbf{Empirical validation.}
    Across three evaluated Qwen student scales and six mathematical reasoning benchmarks, \alg{} achieves the highest macro Pass@8 in all three settings and the highest or second-highest macro Avg@8 among the compared methods.
    These results indicate stronger multi-sample solution coverage while maintaining competitive average accuracy.
\end{itemize}
}

\vspace{-3mm}
\section{Preliminaries}
\label{sec:prelim}
\vspace{-2mm}

\textbf{On-Policy Distillation.}
Let \(\mathcal{D}\) denote the prompt distribution, and let
\(\pi_\theta\) and \(\pi_T\) be the student and teacher policies over a
shared vocabulary \(\mathcal{V}\). Given \(q\sim\mathcal{D}\), OPD samples a
student trajectory \(x=(x_1,\ldots,x_T)\sim\pi_\theta(\cdot\mid q)\) and
queries the teacher on each student-induced prefix \(c_t=(q,x_{<t})\). At
each prefix, it uses the reverse-KL loss
\(\mathcal{L}^{\mathrm{OPD}}_t :=
D_{\mathrm{KL}}(\pi_\theta(\cdot\mid c_t)\|\pi_T(\cdot\mid c_t))\).
Averaging over the student trajectory gives
\begin{equation}
\label{eq:opd}
\mathcal{L}_{\mathrm{OPD}}(\theta)
=
\mathbb{E}_{q\sim\mathcal{D},\,x\sim\pi_\theta(\cdot\mid q)}
\!\left[\frac{1}{T}\sum_{t=1}^{T}\mathcal{L}^{\mathrm{OPD}}_t\right].
\end{equation}
In practice, this objective is estimated from tokens sampled by a frozen
behavior policy \(\pi_{\theta_{\mathrm{old}}}\) and optimized with PPO-style
clipping. Querying the teacher on student-visited prefixes reduces
state-distribution mismatch and provides dense token-level feedback.
However, the mode-seeking reverse KL can under-cover plausible alternatives
when the teacher is uncertain.

\textbf{Entropy-Aware On-Policy Distillation.}
EOPD~\citep{jin2026entropyawareonpolicydistillationlanguage} augments OPD
with mode-covering supervision at positions of high teacher entropy. The raw
entropy and its vocabulary-normalized counterpart are
\begin{equation}
\label{eq:teacher_entropy}
H_T(c_t)
 := -\sum_{v\in\mathcal{V}}\pi_T(v\mid c_t)\log\pi_T(v\mid c_t),
\qquad
\bar H_T(c_t)
 := \frac{H_T(c_t)}{\log|\mathcal{V}|}\in[0,1].
\end{equation}
EOPD gates supervision using \(H_T\), whereas our acquisition score later
uses the bounded, vocabulary-normalized entropy \(\bar H_T\). For efficiency,
let \(S_t^k:=\operatorname{TopK}_k(\pi_T(\cdot\mid c_t))\) denote the
teacher's top-\(k\) candidate set. For \(j\in\{T,\theta\}\), write
\(\bar\pi_j^k(v\mid c_t):=\pi_j(v\mid c_t)/
\sum_{u\in S_t^k}\pi_j(u\mid c_t)\) for the restriction of \(\pi_j\)
renormalized on this candidate set. EOPD approximates forward KL as
\begin{equation}
\label{eq:topk_fkl}
\mathcal{L}^{\mathrm{FKL}}_t
 := \sum_{v\in S_t^k}\bar\pi_T^k(v\mid c_t)
\log\frac{\bar\pi_T^k(v\mid c_t)}{\pi_\theta(v\mid c_t)}.
\end{equation}
Only the teacher target is truncated and renormalized; the student retains
its full-vocabulary probabilities. With entropy threshold \(\tau\) and
forward-KL weight \(\alpha\), the per-token objective is
\begin{equation}
\label{eq:eopd}
\mathcal{L}^{\mathrm{EOPD}}_t
 := \mathcal{L}^{\mathrm{OPD}}_t
+\alpha\,\mathbb{I}\!\left[H_T(c_t)>\tau\right]\mathcal{L}^{\mathrm{FKL}}_t.
\end{equation}
Here, \(\pi_j\) always denotes a full-vocabulary policy, while
\(\bar\pi_j^k\) denotes its top-\(k\)-renormalized shape---a distinction that
separates probability mass from relative shape in our method. EOPD promotes
coverage, but still uses a scalar entropy criterion to decide \emph{where}
to intervene and the uncalibrated teacher prior to determine \emph{what} to
distill; \alg{} revisits both decisions.


\vspace{-2mm}
\section{Methodology}
\vspace{-2mm}

\subsection{Overview}
\label{sec:overview}
\begin{figure*}[t]
    \centering
    \includegraphics[width=\textwidth]{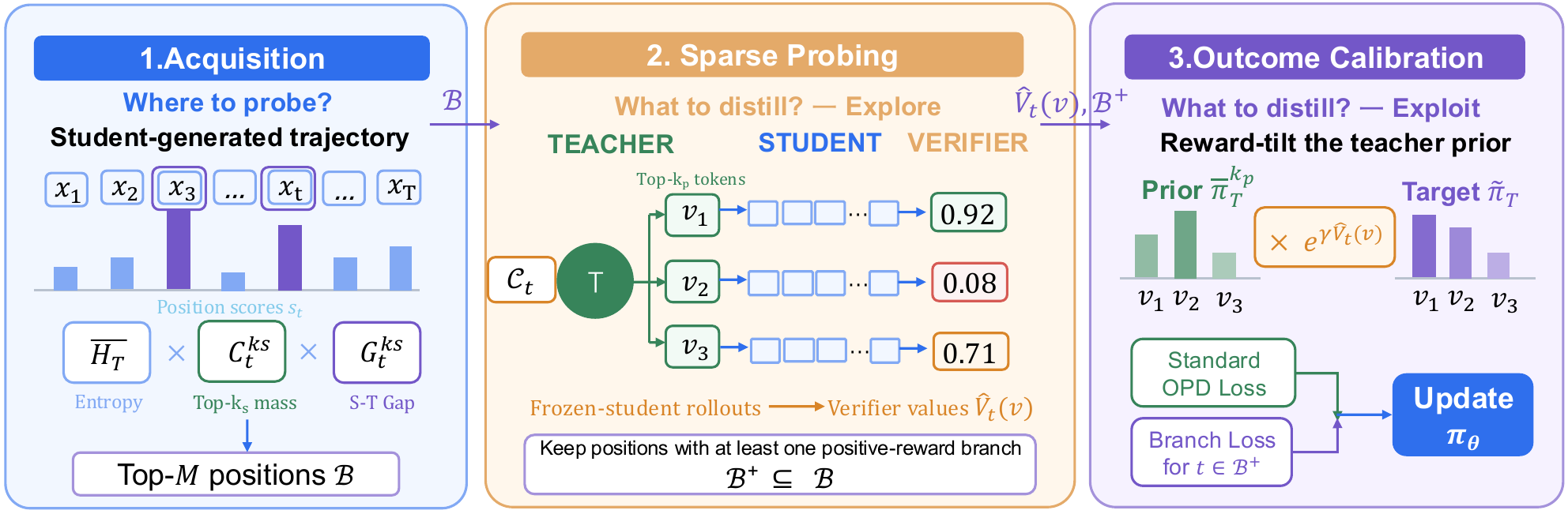}
    \vspace{-3mm}
    \caption{An overview of our \alg{} framework for on-policy distillation.}
    \vspace{-3mm}
    \label{fig:pipeline}
\end{figure*}
\vspace{-5pt}

\alg{} follows a three-stage acquisition--exploration--exploitation procedure that addresses two coupled decisions: \emph{where to probe} and \emph{what to distill}. As illustrated in \Cref{fig:pipeline},
acquisition uses a lightweight position-level score to prioritize where to acquire additional evidence; exploration estimates the downstream values of teacher-proposed next-token candidates through verifier-scored student continuations; and exploitation combines these values with the teacher probabilities to construct outcome-calibrated targets. In this way, acquisition allocates the probing budget, while exploration and exploitation determine how the acquired evidence modifies local supervision.

\begin{algorithm}[t]
\caption{\alg{} Training}
\label{alg:spot}
\begin{algorithmic}[1]
\Require Student $\pi_\theta$, teacher $\pi_T$, verifier $R$, prompt set $\mathcal{D}$, learning rate $\eta$
\Require Position budget $M$, top-$k$ sizes $k_s$ (scoring) and $k_p$ (probing), inverse temperature $\gamma$, branch-loss weight $\beta$
\For{each training iteration}
  \State $\pi_{\theta_{\mathrm{old}}}\leftarrow\pi_\theta$
  \State Rollout buffer $\mathcal{R}\leftarrow\emptyset$
  \For{each prompt $q$ in a batch sampled from $\mathcal{D}$}
    \State Roll out $x\sim\pi_{\theta_{\mathrm{old}}}(\cdot\mid q)$
    \State Query the teacher $\pi_T(\cdot\mid c_t)$ at every prefix $c_t=(q,x_{<t})$
    \State Score every valid position by its position-level acquisition score
      \[
        s_t=\bar{H}_T(c_t)\cdot C_t^{k_s}\cdot G_t^{k_s}
      \]
      \Statex \hspace{\algorithmicindent}using Eqs.~(\ref{eq:spot_concentration})--(\ref{eq:spot_score})
    \State Select sparse probing positions $\mathcal{B}=\operatorname{Top\text{-}}M\bigl(\{s_t\}_{t=1}^{|x|}\bigr)$
    \For{each position $t\in\mathcal{B}$ and candidate $v\in S_t^{k_p}$}
      \State Sample $y\sim\pi_{\theta_{\mathrm{old}}}(\cdot\mid c_t,v)$ and estimate $\hat{V}_t(v)$ using $R(q,x_{<t},v,y)$
    \EndFor
    \State Retain positions with at least one positive-reward candidate as $\mathcal{B}^{+}\subseteq\mathcal{B}$
    \State Form reward-tilted targets $\tilde{\pi}_T(\cdot\mid c_t)$ for $t\in\mathcal{B}^{+}$ using Eq.~(\ref{eq:spot_target})
    \State Add the rollout and its supervision to $\mathcal{R}$
  \EndFor
  \For{each gradient step on a mini-batch sampled from $\mathcal{R}$}
    \State Update $\theta\leftarrow\theta-\eta\nabla_\theta\mathcal{L}$ by minimizing
      \[
        \mathcal{L}=\frac{1}{T}\sum_{t=1}^{T}\mathcal{L}^{\mathrm{OPD}}_t
        +\frac{\beta}{\max\{1,|\mathcal{B}^{+}|\}}\sum_{t\in\mathcal{B}^{+}}\mathcal{L}^{\mathrm{Branch}}_t
      \]
      \Statex \hspace{\algorithmicindent}using Eq.~(\ref{eq:spot_objective})
  \EndFor
\EndFor
\end{algorithmic}
\end{algorithm}

\vspace{-2mm}
\subsection{SPOT: Sparse Probing and Outcome-Calibrated Targets}
\label{sec:spot}
\vspace{-2mm}

\textbf{Position Acquisition: Where to Probe.}
Because probing next-token candidates requires extra rollouts, acquisition prioritizes positions where several teacher candidates are plausible, a small candidate set captures most teacher mass, and the student does not already represent those candidates well. Entropy alone cannot distinguish this pattern from uncertainty spread over a long tail. Recent analyses also suggest that OPD gains depend on student--teacher compatibility and genuinely new teacher information~\citep{li2026rethinkingonpolicydistillationlarge}. We therefore use student--teacher mismatch to signal potential correction needed, while verifier-scored rollouts test the utility of the selected candidates.

Starting from the normalized teacher entropy $\bar{H}_T(c_t)$, let $k_s$
denote the number of teacher candidates used to compute the position-level
acquisition score. We first measure the teacher mass captured by its
top-$k_s$ candidate set:
\begin{equation}
C_t^{k_s} = \sum_{v \in S_t^{k_s}} \pi_T(v\mid c_t).
\label{eq:spot_concentration}
\end{equation}
The quantity \(C_t^{k_s}\) is the teacher probability mass captured by its
top-\(k_s\) candidates. A large value means that a small candidate set
represents most of the teacher distribution. Combined with high entropy,
this favors \emph{positions with multiple plausible candidates} in that set over
positions whose uncertainty is dispersed across a long probability tail.

Teacher-side structure, however, does not imply that the resulting
supervision is useful to a student that may already represent the same
alternatives. We thus measure the student--teacher gap on the same candidate
set:
\begin{equation}
G_t^{k_s} = \lambda_{\mathrm{mass}}(1 - A_t^{k_s}) + \lambda_{\mathrm{shape}} D_{\mathrm{JS}}(\bar{\pi}_T^{k_s} \,\|\, \bar{\pi}_{\theta_{\mathrm{old}}}^{k_s}),
\label{eq:spot_gap}
\end{equation}
where $A_t^{k_s} = \sum_{v \in S_t^{k_s}} \pi_{\theta_{\mathrm{old}}}(v\mid c_t)$ and $\lambda_{\mathrm{mass}}+\lambda_{\mathrm{shape}}=1$. Here $D_{\mathrm{JS}}$ denotes the Jensen--Shannon divergence normalized to $[0,1]$; the two terms capture mass under-coverage and relative-shape mismatch, respectively.

The final acquisition score $s_t$ acts as a soft conjunction: a position receives a high score
only when the teacher presents multiple plausible candidates, the
top-\(k_s\) set captures most teacher probability mass, and the student
either underweights or differently ranks those candidates:
\begin{equation}
s_t = \bar{H}_T(c_t) \cdot C_t^{k_s} \cdot G_t^{k_s}.
\label{eq:spot_score}
\end{equation}
After masking special, padding, whitespace-only, and punctuation-only tokens, we allocate probing to $\mathcal{B}=\operatorname{Top\text{-}}M(\{s_t\}_{t=1}^{|x|})$. Consequently, the rollout budget is reserved for compact teacher alternatives that the current student neither adequately covers nor matches.

\paragraph{Sparse Probing and Outcome Calibration: What to Distill.}
Acquisition prioritizes where to probe, but a teacher's local next-token
probability does not by itself predict downstream success. Prior work
suggests that it can become a weaker guide on deeper student-generated
prefixes~\citep{li2026rethinkingonpolicydistillationlarge}; moreover, a
teacher-preferred token may lead to an incorrect solution, while a
less-preferred candidate may still enable a correct one
~\citep{hou2026uniopdunifyingonpolicydistillation}.
We therefore treat the teacher distribution as a proposal prior (an initial
preference over candidates) and calibrate it using downstream outcomes.

To obtain outcome-grounded evidence, we perform \textit{sparse probing} at each selected position $t\in\mathcal{B}$. Let $k_p$ denote the number of teacher candidates probed at each selected position and $N_p$ the number of continuations sampled per candidate. For each $v\in S_t^{k_p}$, we append $v$ to the prefix $c_t$ and independently sample $N_p$ continuations $y^{(n)}\sim\pi_{\theta_{\mathrm{old}}}(\cdot\mid c_t,v)$, $n=1,\ldots,N_p$, from the frozen behavior policy. These continuations probe the candidate-conditioned branch induced by $v$, whose student-executable value is
\begin{equation}
\hat{V}_t(v)=\mathbb{E}_{y\sim\pi_{\theta_{\mathrm{old}}}(\cdot\mid c_t,v)}\!\left[R(q,x_{<t},v,y)\right],
\label{eq:spot_value}
\end{equation}
which measures whether the current student can turn a teacher-proposed local alternative into a high-reward trajectory. To avoid imposing additional supervision where exploration finds no viable alternative, we retain only positions for which at least one candidate has a positive estimated continuation value; the resulting set is denoted by $\mathcal{B}^{+}\subseteq\mathcal{B}$. Because probing evaluates at most $M\times k_p\times N_p$ candidate continuations per trajectory, its overhead is explicitly controlled by the position, candidate, and per-candidate sampling budgets.

To convert the probed outcomes into actionable supervision, we construct \textit{outcome-calibrated targets} at each retained position $t\in\mathcal{B}^{+}$. Let $\Delta(S_t^{k_p})=\{\rho_t:\rho_t(v)\geq 0,\,\sum_{v\in S_t^{k_p}}\rho_t(v)=1\}$ denote the probability simplex over the teacher candidate set. Rather than collapsing onto the empirically best candidate, we maximize expected downstream value within a KL trust region around the renormalized teacher prior:
\begin{equation}
\max_{\rho_t\in\Delta(S_t^{k_p})}\;\sum_{v\in S_t^{k_p}}\rho_t(v)\hat{V}_t(v)\quad\text{s.t.}\quad D_{\mathrm{KL}}\!\left(\rho_t\,\|\,\bar{\pi}_T^{k_p}(\cdot\mid c_t)\right)\le\epsilon.
\label{eq:spot_target_optimization}
\end{equation}
The derivation is deferred to Appendix~\ref{app:target_theory}. For the corresponding dual parameter $\gamma>0$, which acts as an inverse temperature, the problem admits the closed-form optimizer
\begin{equation}
\tilde{\pi}_T(v\mid c_t)=\frac{\bar{\pi}_T^{k_p}(v\mid c_t)\exp\!\left(\gamma\hat{V}_t(v)\right)}{\sum_{u\in S_t^{k_p}}\bar{\pi}_T^{k_p}(u\mid c_t)\exp\!\left(\gamma\hat{V}_t(u)\right)}.
\label{eq:spot_target}
\end{equation}

In log space, \eqref{eq:spot_target} reads $\log\tilde{\pi}_T(v\mid c_t)=\log\bar{\pi}_T^{k_p}(v\mid c_t)+\gamma\hat{V}_t(v)-\log Z_t$, exposing each target log-probability as \emph{teacher log-probability plus an outcome-grounded continuation-value bonus}, up to the normalizer $Z_t$. Subtracting this expression for any $u,v\in S_t^{k_p}$ yields the pairwise log-odds decomposition
\begin{equation}
\log\frac{\tilde{\pi}_T(v\mid c_t)}{\tilde{\pi}_T(u\mid c_t)}=\log\frac{\bar{\pi}_T^{k_p}(v\mid c_t)}{\bar{\pi}_T^{k_p}(u\mid c_t)}+\gamma\!\left(\hat{V}_t(v)-\hat{V}_t(u)\right).
\label{eq:spot_pairwise_odds}
\end{equation}
This identity makes the division of roles explicit: the teacher supplies the prior odds, while the verifier adds a correction proportional to relative downstream value. Conditional on $t\in\mathcal{B}^{+}$, only value differences affect the target, making it invariant to uniform reward shifts. Together with the $\mathcal{B}^{+}$ gate, outcome feedback first tests absolute viability and then corrects relative preference. As $\gamma\to0$, the target recovers the teacher prior; larger $\gamma$ increasingly favors candidates with higher estimated continuation values.

Finally, we augment OPD with outcome-calibrated local supervision at verifier-supported positions:
\begin{equation}
\mathcal{L} = \frac{1}{T}\sum_{t=1}^{T} \mathcal{L}_t^{\mathrm{OPD}} + \frac{\beta}{\max\{1,|\mathcal{B}^{+}|\}}\sum_{t \in \mathcal{B}^{+}} \mathcal{L}_t^{\mathrm{Branch}},
\label{eq:spot_objective}
\end{equation}
where
\(\mathcal{L}_t^{\mathrm{Branch}}
=
-\sum_{v\in S_t^{k_p}}
\tilde{\pi}_T(v\mid c_t)\log\pi_\theta(v\mid c_t)\).
The denominator averages this auxiliary loss over retained positions, and
\(\beta\) controls its strength relative to trajectory-wide OPD.
When $\mathcal{B}^{+}=\varnothing$, the branch term vanishes and the objective reduces to standard OPD.

\vspace{-2mm}
\section{EXPERIMENTS}
\label{sec:experiments}
\vspace{-2mm}


For a rigorous, controlled comparison of \alg{}, we follow the setup of recent OPD work~\citep{jin2026entropyawareonpolicydistillationlanguage}. Further details appear in Appendix~\ref{app:more_experiment_details}.

\vspace{-1.5mm}
\subsection{Experimental Settings}
\label{subsec:exp_settings}
\vspace{-1.5mm}

\textbf{Models and Training Data.} We use Qwen3-8B, with thinking mode disabled, as the common teacher for all distillation methods, and Qwen3-0.6B-Base, Qwen3-1.7B-Base, and Qwen3-4B-Base as students \citep{yang2025qwen3technicalreport}. The 0.6B and 1.7B students are trained on MATH \citep{hendrycks2021measuringmathematicalproblemsolving}; the 4B student uses the more challenging DAPO dataset \citep{yu2025dapoopensourcellmreinforcement}.

\textbf{Baselines.} We compare \alg~with four baselines:
\vspace{-3mm}
{
\setlength{\leftmargini}{1.5em}
\begin{itemize}
    \item \textbf{KD} \citep{hinton2015distillingknowledgeneuralnetwork,kim2016sequencelevelknowledgedistillation}: off-policy distillation using forward KL and cross-entropy on teacher-generated data.
    \item \textbf{OPD} \citep{agarwal2024onpolicydistillationlanguagemodels,li2026rethinkingonpolicydistillationlarge}: on-policy training on student rollouts with per-token reverse-KL supervision.
    \item \textbf{GRPO} \citep{shao2024deepseekmathpushinglimitsmathematical}: outcome-based RL using group-relative advantages from verifiable rewards, without teacher supervision.
    \item \textbf{EOPD} \citep{jin2026entropyawareonpolicydistillationlanguage}: OPD augmented with a top-$k$ forward-KL term at positions selected by a fixed teacher-entropy threshold.
\end{itemize}
}
\textbf{Evaluation.} We evaluate all models zero-shot with the same prompt template and answer verifier on MATH-500 \citep{hendrycks2021measuringmathematicalproblemsolving,lightman2023letsverifystepstep}, AIME 2024 \citep{aime24}, AIME 2025 \citep{aime25}, AMC 2023 \citep{yang2024qwen2}, Minerva Math \citep{lewkowycz2022solvingquantitativereasoningproblems}, and HMMT 2025 \citep{balunović2026matharenaevaluatingllmsuncontaminated}. We sample eight responses per problem with temperature 1.0, top-$p$ sampling with $p=0.8$, and a maximum response length of 8192 tokens. \textbf{Avg@8} is the mean response accuracy, and \textbf{Pass@8} is the fraction of problems solved at least once. We report per-benchmark scores and their unweighted macro-average across the six benchmarks.

\vspace{-1mm}
\subsection{Main Results}
\vspace{-1mm}
\textbf{Consistent gains across student scales and training sets.}
As shown in Table~\ref{tab:main_results}, \alg{} achieves the best macro Pass@8 at all three evaluated student scales and the best or second-best macro Avg@8. Relative to standard OPD, \alg{} improves macro Avg@8 by 0.47--1.48 points and macro Pass@8 by 4.55--5.28 points. Relative to EOPD, the closest uncertainty-aware OPD baseline, \alg{} improves macro Avg@8 by 0.29--0.68 points and macro Pass@8 by 2.49--3.19 points. These gains persist from 0.6B to 4B students and across MATH and DAPO-Math-14k training, demonstrating that the benefit is not tied to a particular capacity or training regime.

\textbf{Broader multi-sample coverage without sacrificing average accuracy.}
Avg@8 estimates per-sample correctness, whereas Pass@8 measures whether eight samples reach at least one correct solution. The markedly larger Pass@8 gains, together with preserved or improved Avg@8, reveal a favorable coverage--quality profile: \alg{} places meaningful probability on viable alternatives without diluting an individual attempt. This directly matches \alg{}'s goal of improving coverage without sacrificing average accuracy. The comparison with EOPD is consistent with our motivating distinction: teacher entropy alone cannot determine whether uncertainty is concentrated in a compact candidate set, whether the student already represents those candidates, or which candidates lead to successful student continuations. \alg{} addresses these gaps by using normalized teacher entropy, top-$k_s$ probability mass, and student--teacher mismatch to allocate the probing budget, then using verifier-scored continuation values to construct outcome-calibrated local targets. 
\alg{} therefore selectively expands coverage toward useful alternative reasoning modes rather than stylistic or erroneous variation.

\vspace{-6pt}
\begin{table*}[t]
\centering
\definecolor{myskyblue}{HTML}{EBF3FC}
\colorlet{mylightgreen}{green!10}

\newcommand{\ourscell}[1]{\cellcolor{myskyblue}#1}
\newcommand{\summarycell}[1]{\cellcolor{mylightgreen}#1}

\caption{
Main results on six mathematical reasoning benchmarks.
\textit{Avg.} is the unweighted mean over the six benchmarks, computed before rounding.
\textbf{Bold} and \underline{underlined} entries denote the best and second-best results within each student setting, respectively.
}
\label{tab:main_results}

\renewcommand{\arraystretch}{1.2}

\resizebox{\textwidth}{!}{
\begin{tabular}{@{} l ccccccc ccccccc @{}}
\toprule
\multirow{2}{*}{Method} 
& \multicolumn{7}{c}{\textbf{Avg@8}} 
& \multicolumn{7}{c}{\textbf{Pass@8}} \\
\cmidrule(lr){2-8} \cmidrule(lr){9-15}
& MATH500 & AMC23 & Minerva & HMMT & AIME24 & AIME25 & Avg.
& MATH500 & AMC23 & Minerva & HMMT & AIME24 & AIME25 & Avg. \\
\midrule

\multicolumn{15}{c}{\textit{Student Model: \textbf{Qwen3-0.6B-Base}} (Training Data: MATH)} \\
\midrule
KD   
& 47.80 & 23.43 & 14.34 & 0.21 & 2.08 & 0.83 & \summarycell{14.78}
& 69.60 & 52.50 & 30.88 & 1.67 & 6.67 & 6.67 & \summarycell{28.00} \\

GRPO 
& \textbf{53.33} & \underline{28.13} & \textbf{16.41} & 1.04 & \underline{4.58} & 0.83 & \summarycell{\textbf{17.39}}
& 74.40 & 55.00 & 32.72 & 3.33 & \underline{10.00} & \underline{10.00} & \summarycell{30.91} \\

OPD  
& 50.10 & 24.69 & \underline{16.04} & 0.42 & 2.50 & \underline{1.25} & \summarycell{15.83}
& 73.20 & \underline{57.50} & 31.25 & 1.67 & \underline{10.00} & 6.67 & \summarycell{30.05} \\

EOPD 
& \underline{50.50} & 27.81 & 15.95 & \underline{1.46} & 4.17 & \underline{1.25} & \summarycell{16.86}
& \underline{75.00} & 55.00 & \underline{33.46} & \underline{5.00} & \textbf{13.33} & 6.67 & \summarycell{\underline{31.41}} \\

\ourscell{\textbf{SPOT (Ours)}} 
& \ourscell{\underline{50.50}}
& \ourscell{\textbf{28.44}} 
& \ourscell{15.99} 
& \ourscell{\textbf{1.88}} 
& \ourscell{\textbf{5.00}} 
& \ourscell{\textbf{2.08}} 
& \summarycell{\underline{17.31}}
& \ourscell{\textbf{76.40}}
& \ourscell{\textbf{60.00}} 
& \ourscell{\textbf{34.56}} 
& \ourscell{\textbf{6.67}} 
& \ourscell{\textbf{13.33}} 
& \ourscell{\textbf{16.67}} 
& \summarycell{\textbf{34.60}} \\

\midrule

\multicolumn{15}{c}{\textit{Student Model: \textbf{Qwen3-1.7B-Base}} (Training Data: MATH)} \\
\midrule
KD   
& 62.85 & 37.81 & 27.07 & 1.25 & 10.42 & 3.33 & \summarycell{23.79}
& 84.20 & 70.00 & 44.12 & \underline{3.33} & 20.00 & 16.67 & \summarycell{39.72} \\

GRPO 
& 67.00 & 38.13 & \underline{28.17} & \underline{1.88} & 9.17 & 5.42 & \summarycell{24.96}
& 84.00 & \underline{72.50} & \textbf{48.16} & \textbf{6.67} & 20.00 & 16.67 & \summarycell{41.33} \\

OPD  
& 67.03 & 39.06 & 27.62 & 1.46 & 8.33 & \underline{6.25} & \summarycell{24.96}
& \underline{84.80} & 70.00 & 47.06 & \underline{3.33} & 20.00 & 16.67 & \summarycell{40.31} \\

EOPD 
& \textbf{67.68} & \textbf{40.31} & 27.90 & 1.67 & \underline{10.83} & \textbf{7.50} & \summarycell{\underline{25.98}}
& 83.60 & \textbf{75.00} & 44.49 & \textbf{6.67} & \underline{26.67} & \underline{20.00} & \summarycell{\underline{42.74}} \\

\ourscell{\textbf{SPOT (Ours)}} 
& \ourscell{\underline{67.43}} 
& \ourscell{\underline{39.69}} 
& \ourscell{\textbf{28.45}} 
& \ourscell{\textbf{2.08}} 
& \ourscell{\textbf{12.50}} 
& \ourscell{\textbf{7.50}} 
& \summarycell{\textbf{26.27}}
& \ourscell{\textbf{87.80}} 
& \ourscell{\textbf{75.00}} 
& \ourscell{\underline{47.43}} 
& \ourscell{\textbf{6.67}} 
& \ourscell{\textbf{33.33}} 
& \ourscell{\textbf{23.33}} 
& \summarycell{\textbf{45.59}} \\

\midrule

\multicolumn{15}{c}{\textit{Student Model: \textbf{Qwen3-4B-Base}} (Training Data: DAPO-Math-14k)} \\
\midrule
KD   
& 74.73 & 48.13 & 34.65 & 3.13 & 12.50 & 12.08 & \summarycell{30.87}
& 92.20 & 80.00 & 51.47 & 13.33 & 26.67 & 23.33 & \summarycell{47.83} \\

GRPO 
& 79.18 & 53.13 & \textbf{39.84} & 2.71 & 14.17 & 13.33 & \summarycell{33.73}
& 90.00 & 80.00 & \underline{52.21} & 11.67 & 26.67 & 26.67 & \summarycell{47.87} \\

OPD  
& \textbf{80.38} & \underline{55.94} & 37.64 & 5.83 & \textbf{17.92} & 16.25 & \summarycell{\underline{35.66}}
& 93.00 & 80.00 & 51.10 & 13.33 & \underline{30.00} & 30.00 & \summarycell{49.57} \\

EOPD 
& 79.73 & 55.63 & 36.08 & \underline{7.08} & \underline{17.50} & \underline{16.67} & \summarycell{35.45}
& \underline{93.20} & \underline{82.50} & 51.84 & \underline{16.67} & \textbf{33.33} & \underline{33.33} & \summarycell{\underline{51.81}} \\

\ourscell{\textbf{SPOT (Ours)}} 
& \ourscell{\underline{79.88}} 
& \ourscell{\textbf{56.25}} 
& \ourscell{\underline{37.73}} 
& \ourscell{\textbf{7.50}} 
& \ourscell{\underline{17.50}} 
& \ourscell{\textbf{17.92}} 
& \summarycell{\textbf{36.13}}
& \ourscell{\textbf{93.40}} 
& \ourscell{\textbf{85.00}} 
& \ourscell{\textbf{54.04}} 
& \ourscell{\textbf{23.33}} 
& \ourscell{\textbf{33.33}} 
& \ourscell{\textbf{36.67}} 
& \summarycell{\textbf{54.30}} \\

\bottomrule
\end{tabular}
}
\end{table*}

\vspace{-0.5mm}
\subsection{Ablation Study}
\label{subsec:ablation}
\vspace{-1mm}

\subsubsection{Effect of the Branch-Acquisition Score}
\label{subsec:ablation_acquisition}
\vspace{-1mm}


Because verifier-based candidate probing is costly, \alg{} restricts it to the
top-$M$ positions. Here, \(H\) denotes normalized teacher entropy, \(C\)
denotes the teacher mass captured by its top-$k_s$ candidates, and \(G\)
denotes the student--teacher gap combining mass undercoverage and JS shape
mismatch. We compare five variants: \texttt{H} uses entropy alone;
\texttt{HC} adds \(C\); \texttt{HG} adds \(G\) but omits \(C\);
\texttt{HC-Mass} uses \(H\), \(C\), and only the mass-undercoverage term;
and \texttt{Full} uses all components. All runs share the teacher, data,
schedule, probing budget, and evaluation protocol.

\begin{table}[t]
\centering
\definecolor{scorefull}{HTML}{EEF5FC}
\caption{\textbf{Branch-acquisition score ablation.}
Avg@8 / Pass@8 (\%) for Qwen3-0.6B-Base trained on MATH. \textit{Macro}
denotes the unweighted mean; best results are bold and \textsc{Full} is shaded.}
\label{tab:ablation_acquisition}
\small
\setlength{\tabcolsep}{4pt}
\renewcommand{\arraystretch}{1.15}
\begin{tabular}{@{}lccccc@{}}
\toprule
& \multicolumn{5}{c}{\textbf{Avg@8 / Pass@8}} \\
\cmidrule(lr){2-6}
Acquisition score & MATH500 & AMC23 & AIME24 & AIME25 & Macro \\
\midrule
\texttt{H}: $\bar H_T$
& \underline{50.50} / \underline{76.20}
& 25.00 / \underline{60.00}
& 1.67 / 3.33
& 0.42 / 3.33
& 19.40 / 35.72 \\
\texttt{HC}: $\bar H_T C_t^{k_s}$
& 50.30 / 73.20
& 27.50 / 57.50
& 1.67 / 3.33
& 0.00 / 0.00
& 19.87 / 33.51 \\
\texttt{HG}: $\bar H_T G_t^{k_s}$
& 50.00 / 74.20
& \underline{28.13} / 55.00
& \underline{2.08} / \underline{10.00}
& \underline{1.25} / \underline{10.00}
& 20.37 / 37.30 \\
\texttt{HC-Mass}: $\bar H_T C_t^{k_s}(1-A_t^{k_s})$
& \best{50.75} / 75.80
& 27.81 / \best{62.50}
& \underline{2.08} / 6.67
& \underline{1.25} / 6.67
& \underline{20.47} / \underline{37.91} \\
\textbf{\texttt{Full}}: $\bar H_T C_t^{k_s}G_t^{k_s}$
& \underline{50.50} / \best{76.40}
& \best{28.44} / \underline{60.00}
& \best{5.00} / \best{13.33}
& \best{2.08} / \best{16.67}
& \best{21.51} / \best{41.60} \\
\bottomrule
\end{tabular}
\end{table}


Table~\ref{tab:ablation_acquisition} shows that \texttt{Full} achieves the
best macro Avg@8/Pass@8 (21.51/41.60), with its larger advantage in Pass@8
suggesting improved solution coverage. \texttt{HG} underperforms
\texttt{Full} on both metrics across all benchmarks, supporting \(C_t^{k_s}\)
as a reliable filter against diffuse teacher uncertainty. \texttt{HC}, which omits
the student--teacher gap, also consistently underperforms \texttt{Full}, supporting the inclusion of \(G_t^{k_s}\). \texttt{HC-Mass} is the
strongest partial variant but remains below \texttt{Full}, suggesting that
mass undercoverage and JS shape mismatch capture complementary student
deficits. Overall, the results support the multiplicative score as a soft
conjunction of teacher ambiguity, top-$k_s$ mass.

\vspace{-1.5mm}
\subsubsection{Effect of Verifier-Guided Calibration}
\label{subsec:ablation_verifier_guidance}
\vspace{-1.5mm}
Verifier-guided calibration converts the teacher prior into an outcome-aware
branch target. To isolate this component, we remove reward tilting and
positive-reward gating while retaining \textsc{Full} acquisition and the
branch-distillation objective. This ablation therefore uses the
uncalibrated teacher proposal prior at all probed positions.
All other training and evaluation settings
match the Qwen3-1.7B-Base main experiment.
Table~\ref{tab:ablation_verifier_guidance} shows that the full configuration
improves macro Avg@8/Pass@8 by 3.21/7.38 points. Pass@8 increases on all four
benchmarks, with the largest gain on AIME24 ($+13.33$). On MATH500, Avg@8
changes by only $+0.13$ while Pass@8 rises by $+2.00$, suggesting that
verifier-guided calibration primarily improves solution coverage in this
setting.

\begin{table}[t]
\centering
\definecolor{vgfull}{HTML}{EEF5FC}
\caption{\textbf{Verifier-guidance ablation.}
Avg@8 / Pass@8 (\%) for Qwen3-1.7B-Base trained on MATH. \textit{Macro}
denotes the unweighted mean; best results are bold and the full model is shaded.}
\label{tab:ablation_verifier_guidance}
\small
\setlength{\tabcolsep}{4pt}
\renewcommand{\arraystretch}{1.15}
\begin{tabular}{@{}lccccc@{}}
\toprule
& \multicolumn{5}{c}{\textbf{Avg@8 / Pass@8}} \\
\cmidrule(lr){2-6}
Variant & MATH500 & AMC23 & AIME24 & AIME25 & Macro \\
\midrule
w/o verifier guidance
& 67.30 / 85.80
& 34.06 / 67.50
& 8.33 / 20.00
& 4.58 / 16.67
& 28.57 / 47.49 \\
\rowcolor{vgfull}
\textbf{\alg{} (Full)}
& \best{67.43} / \best{87.80}
& \best{39.69} / \best{75.00}
& \best{12.50} / \best{33.33}
& \best{7.50} / \best{23.33}
& \best{31.78} / \best{54.87} \\
\bottomrule
\end{tabular}
\end{table}

\vspace{-1.5mm}
\subsubsection{Scaling with the Evaluation-Time Sampling Budget}
\label{subsec:ablation_metric_k}
\vspace{-1.5mm}

To test whether \alg{} improves coverage beyond a fixed Pass@8 protocol,
we vary $k\in\{4,8,16,32,64\}$ on AIME24, AIME25, and AMC23, holding
decoding fixed. Here, $k$ denotes responses per problem.
\begin{figure}[t]
    \centering
    \includegraphics[width=0.78\linewidth]{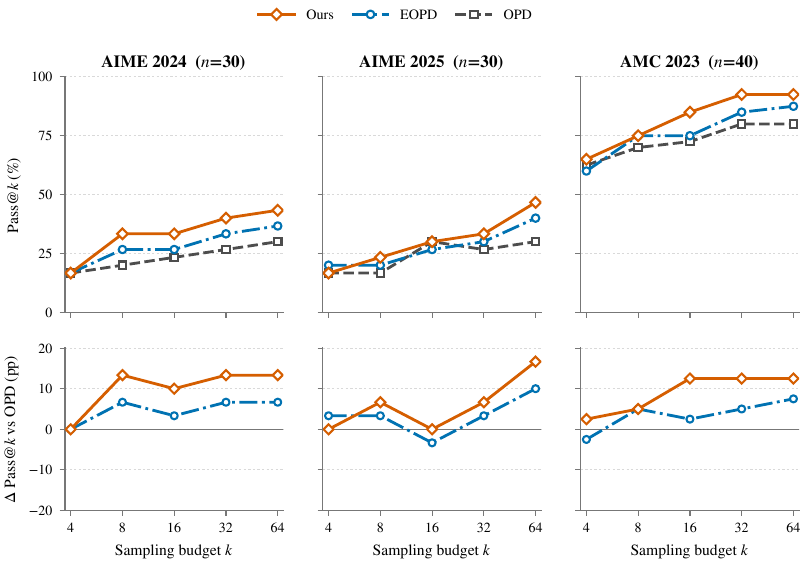}
    \vspace{-8pt}
    \caption{
    Pass@$k$ and its gain over OPD on AIME 2024, AIME 2025,
    and AMC 2023. Each point is read from the corresponding summary of a
    separately generated $k$-sample run.}
    \label{fig:ablation_metric_k}
\end{figure}
Figure~\ref{fig:ablation_metric_k} reveals a budget-dependent separation.
The methods are close at $k=4$, with EOPD leading on AIME25, whereas
\alg{} ranks first or ties from $k=8$ onward.
At $k=64$, \alg{} retains
12.50--16.67-point gains over OPD, showing that its coverage advantage
extends beyond Pass@8 and remains visible as the sampling budget grows.

\vspace{-1.5mm}
\subsubsection{Effect of Branch-Distillation Weight}
\label{subsec:ablation_beta}
\vspace{-1.5mm}

The coefficient $\beta$ trades off trajectory-wide OPD against
outcome-calibrated branch supervision. 
\begin{figure}[t]
    \centering
    \includegraphics[width=0.8\linewidth]{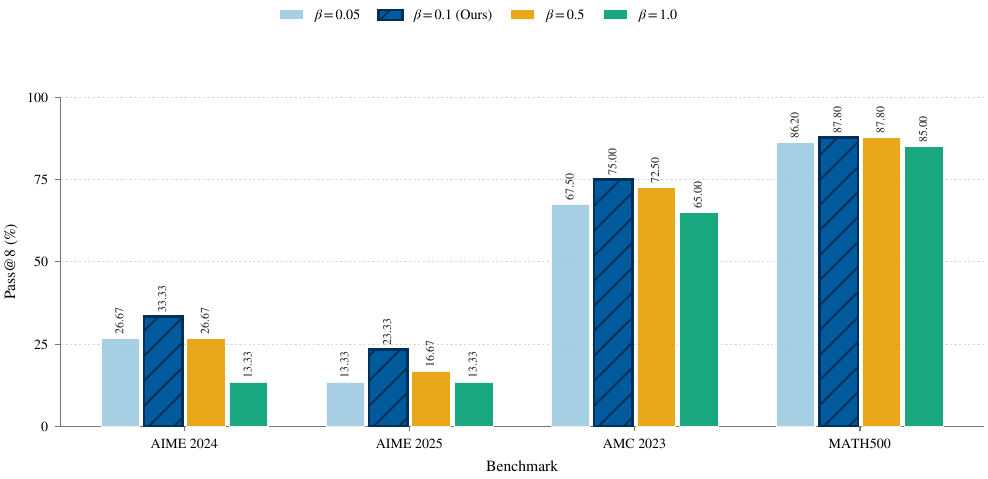}
    \caption{\textbf{Sensitivity to the branch-distillation weight.}
    Pass@8 of Qwen3-1.7B-Base for $\beta\in\{0.05,0.1,0.5,1.0\}$.
    The main configuration $\beta=0.1$ is marked as Ours.}
    \label{fig:ablation_beta}
\end{figure}
Among the tested values in Figure \ref{fig:ablation_beta}, $\beta=0.1$ leads the three competition benchmarks
and ties $\beta=0.5$ on MATH500. The competition benchmarks show greater
observed sensitivity, with their mean Pass@8 lowest at $\beta=1.0$. This
pattern is consistent with branch feedback serving as a calibrated correction
to trajectory-wide supervision: weaker weighting may underuse local outcome
signals, whereas stronger weighting may overemphasize sparse branch targets.
We therefore use $\beta=0.1$ as a
practical default.

\vspace{-1.5mm}
\subsubsection{Out-of-Domain Generalization}
\label{subsec:ood_generalization}
\vspace{-1.5mm}
Although trained exclusively on MATH, \alg{} generalizes beyond the training
distribution, with its clearest gains on deliberative reasoning while remaining
competitive on broader knowledge and instruction following (Appendix~
\ref{app:ood_generalization_results}, Table~\ref{tab:ood_results}). This pattern
suggests that outcome-calibrated branch supervision transfers reusable decision
behavior rather than merely fitting math-specific solution templates. The
non-uniform gains further indicate that branch calibration complements, rather
than replaces, token-level uncertainty transfer and domain-aligned supervision.

\vspace{-1.5mm}
\subsubsection{Consistency Across Model Families}
\label{subsec:ablation_model_family}
\vspace{-1.5mm}

The Llama results (Appendix~\ref{app:model_family_results},
Table~\ref{tab:ablation_model_family}) preserve the main Qwen pattern:
\alg{} leads both macro metrics, with a clearer advantage in Pass@8 than
Avg@8. Together with the Qwen results across model scales and training sets,
this supports the portability of outcome-calibrated branch supervision beyond
a single architecture and suggests that its primary benefit is consistently broader coverage
of viable reasoning paths. 

\vspace{-2mm}
\section{Related Work}
\vspace{-2mm}
\textbf{Knowledge Distillation.}
Knowledge distillation transfers capabilities from a high-capacity teacher to a compact student by matching teacher predictive distributions or imitating teacher-generated sequences
\citep{hinton2015distillingknowledgeneuralnetwork,kim2016sequencelevelknowledgedistillation}.
Recent work improves this paradigm through alternative divergence objectives, contrastive learning, and selective supervision based on sample difficulty or token importance
\citep{ko2024distillmstreamlineddistillationlarge,ko2025distillm2contrastiveapproachboosts,he2025kd,guo2025learningfocuscausalattention}.
Despite their effectiveness, these methods train on contexts that are not sampled from the current student policy, creating exposure bias and compounding errors when the student conditions on its own predictions at inference time
\citep{bengio2015scheduledsamplingsequenceprediction,ranzato2016sequenceleveltrainingrecurrent}.

\textbf{On-Policy Distillation for Reasoning Models.}
On-policy distillation trains students on self-generated rollouts with dense token-level teacher feedback, mitigating exposure bias and providing finer credit assignment than outcome-only RL
\citep{agarwal2024onpolicydistillationlanguagemodels,gu2026minillmonpolicydistillationlarge,lu2025onpolicydistillation}.
Recent work has broadened OPD along four directions:
teacher access and transfer scope, through reward extrapolation, black-box teachers, and multi-teacher distillation
\citep{yang2026learningteachergeneralizedonpolicy,ye2026blackboxonpolicydistillationlarge,hou2026uniopdunifyingonpolicydistillation};
adaptive supervision, through uncertainty- or disagreement-aware token selection, teachability modeling, and trajectory-aware guidance
\citep{jin2026entropyawareonpolicydistillationlanguage,xu2026tiptokenimportanceonpolicy,wang2026disagreementlearnabletokenteachability,jiang2026bridgingreasoningtrajectoriesonpolicy};
representation and training efficiency, through hidden-state alignment and offline teacher scoring
\citep{yang2026oprdonpolicyrepresentationdistillation,wu2026lightningopdefficientposttraining,ziheng2026moreearlystoppingrollout};
and self-distillation, where privileged or auxiliary context induces internal teachers for knowledge internalization
\citep{zhao2026selfdistilledreasoneronpolicyselfdistillation,ye2026onpolicycontextdistillationlanguage}.
Together, these advances have established OPD as a general post-training paradigm, with adoption in large-scale systems such as Qwen3 and DeepSeek-V4
\citep{yang2025qwen3technicalreport,deepseekai2026deepseekv4highlyefficientmilliontoken}.

\vspace{-2mm}
\section{Conclusion}
\vspace{-2mm}


Standard reverse-KL OPD can under-cover plausible alternatives, while teacher
entropy alone cannot determine where limited probing is most useful or which
candidates lead to successful student continuations. We introduced \alg{},
which uses normalized teacher entropy, top-$k_s$ probability mass, and
student--teacher mismatch to allocate a sparse probing budget, then uses
verifier-scored continuation values to construct KL-regularized,
outcome-calibrated targets anchored to the teacher proposal prior. Across the
evaluated student scales, training sets, reasoning benchmarks, and model
families, \alg{} achieves the highest macro Pass@8 in every setting and the
highest or second-highest macro Avg@8, with controlled ablations supporting
both selective acquisition and verifier guidance. 
Overall, our results support separating where to probe from what to distill: uncertainty and student mismatch guide probing, while downstream outcomes calibrate supervision.

\bibliography{iclr2026_conference}
\bibliographystyle{iclr2026_conference}

\appendix
\newpage
\section{Derivation and Properties of Outcome-Calibrated Targets}
\label{app:target_theory}

This appendix derives the outcome-calibrated target and establishes several properties of its value--prior tradeoff. Fix a retained position \(t\in\mathcal{B}^{+}\), and abbreviate
\[
\mathcal{S}=S_t^{k_p},\qquad
p(v)=\bar{\pi}_T^{k_p}(v\mid c_t),\qquad
V(v)=\hat V_t(v).
\]
Here, \(V\) is treated as the branch-value estimate obtained after probing. The candidate set \(\mathcal{S}\) is finite, and \(p(v)>0\) for every \(v\in\mathcal{S}\). We write \(\Delta(\mathcal{S})\) for the probability simplex over \(\mathcal{S}\).

\subsection{Closed-Form Derivation and Trust-Region Equivalence}
\label{app:target_derivation}

For any inverse-temperature parameter \(\gamma>0\), consider the KL-regularized problem
\begin{equation}
\max_{\rho\in\Delta(\mathcal{S})}
\left\{
\mathbb{E}_{v\sim\rho}[V(v)]
-\frac{1}{\gamma}D_{\mathrm{KL}}(\rho\|p)
\right\}.
\label{eq:app_regularized_target}
\end{equation}
The first term favors branches with high estimated downstream value, while the second anchors the target to the teacher prior.

\begin{proposition}[Closed-form optimizer]
\label{prop:closed_form_target}
The objective in Eq.~(\ref{eq:app_regularized_target}) has the unique optimizer
\begin{equation}
\rho_\gamma(v)
=
\frac{p(v)\exp\left(\gamma V(v)\right)}
{\sum_{u\in\mathcal{S}}p(u)\exp\left(\gamma V(u)\right)}.
\label{eq:app_closed_form_target}
\end{equation}
\end{proposition}

\begin{proof}
Let \(Z(\gamma)=\sum_{u\in\mathcal{S}}p(u)\exp(\gamma V(u))\). For any \(\rho\in\Delta(\mathcal{S})\),
\begin{align}
D_{\mathrm{KL}}(\rho\|\rho_\gamma)
&=
\sum_{v\in\mathcal{S}}\rho(v)
\log\frac{\rho(v)Z(\gamma)}
{p(v)\exp(\gamma V(v))} \nonumber\\
&=
D_{\mathrm{KL}}(\rho\|p)
-\gamma\mathbb{E}_{\rho}[V]
+\log Z(\gamma).
\end{align}
Rearranging gives the Gibbs variational identity
\begin{equation}
\mathbb{E}_{\rho}[V]
-\frac{1}{\gamma}D_{\mathrm{KL}}(\rho\|p)
=
\frac{\log Z(\gamma)}{\gamma}
-\frac{1}{\gamma}D_{\mathrm{KL}}(\rho\|\rho_\gamma).
\label{eq:app_gibbs_identity}
\end{equation}
The right-hand side is maximized if and only if
\(D_{\mathrm{KL}}(\rho\|\rho_\gamma)=0\), or equivalently
\(\rho=\rho_\gamma\).
\end{proof}

The regularized and trust-region views describe the same solution path. Define the radius induced by \(\gamma\) as
\begin{equation}
\epsilon(\gamma)
=D_{\mathrm{KL}}(\rho_\gamma\|p).
\label{eq:app_induced_radius}
\end{equation}

\begin{proposition}[Trust-region equivalence]
\label{prop:trust_region_equivalence}
For every \(\gamma>0\), \(\rho_\gamma\) is the unique solution of
\begin{equation}
\max_{\rho\in\Delta(\mathcal{S})}\mathbb{E}_{\rho}[V]
\quad\textnormal{s.t.}\quad
D_{\mathrm{KL}}(\rho\|p)\leq\epsilon(\gamma).
\label{eq:app_trust_region}
\end{equation}
Conversely, if the KL constraint is active and its KKT multiplier is \(\eta>0\), the solution has the form in Eq.~(\ref{eq:app_closed_form_target}) with \(\gamma=1/\eta\).
\end{proposition}

\begin{proof}
Suppose a feasible \(\rho\) attained a larger expected value than \(\rho_\gamma\). Since
\(D_{\mathrm{KL}}(\rho\|p)\leq
D_{\mathrm{KL}}(\rho_\gamma\|p)\), it would also attain a strictly larger objective in Eq.~(\ref{eq:app_regularized_target}), contradicting Proposition~\ref{prop:closed_form_target}. The same argument, together with the uniqueness of the regularized optimizer, excludes any distinct feasible optimizer with equal expected value.

For the converse direction, introduce multipliers \(\eta\geq0\) and \(\lambda\in\mathbb{R}\). The Lagrangian for Eq.~(\ref{eq:app_trust_region}) is
\[
\mathcal{J}(\rho,\eta,\lambda)
=\mathbb{E}_{\rho}[V]
-\eta\left(D_{\mathrm{KL}}(\rho\|p)-\epsilon\right)
+\lambda\left(\sum_{v\in\mathcal{S}}\rho(v)-1\right).
\]
When the constraint is active, stationarity with respect to \(\rho(v)\) yields
\[
V(v)-\eta\left(\log\frac{\rho(v)}{p(v)}+1\right)+\lambda=0.
\]
Normalizing the resulting probabilities gives
\(\rho(v)\propto p(v)\exp(V(v)/\eta)\), which is
Eq.~(\ref{eq:app_closed_form_target}) with \(\gamma=1/\eta\).
\end{proof}

Thus, a shared \(\gamma\) parameterizes the regularized solution directly and induces a position-dependent trust-region radius
\(\epsilon_t(\gamma)=D_{\mathrm{KL}}\left(\rho_{\gamma,t}\mathbin{\Vert}\bar{\pi}_T^{k_p}(\cdot\mid c_t)\right)\).
For a prescribed active radius, the corresponding inverse temperature is the reciprocal of the KL multiplier rather than the multiplier itself.

\subsection{Geometry of the Calibration Path}
\label{app:calibration_geometry}

The log-partition function provides a compact characterization of how \(\gamma\) trades off estimated branch value against deviation from the teacher.
Let
\[
F(\gamma)=\log Z(\gamma)
=\log\sum_{v\in\mathcal{S}}p(v)\exp(\gamma V(v)).
\]

\begin{proposition}[Monotone value--prior tradeoff]
\label{prop:monotone_tradeoff}
Along the calibration path \(\{\rho_\gamma\}_{\gamma\geq0}\), with \(\rho_0=p\) defined by continuity,
\begin{align}
\frac{\mathrm{d}}{\mathrm{d}\gamma}
\mathbb{E}_{\rho_\gamma}[V]
&=\operatorname{Var}_{\rho_\gamma}(V)\geq0,
\label{eq:app_value_derivative}\\
\frac{\mathrm{d}}{\mathrm{d}\gamma}
D_{\mathrm{KL}}(\rho_\gamma\|p)
&=\gamma\operatorname{Var}_{\rho_\gamma}(V)\geq0.
\label{eq:app_kl_derivative}
\end{align}
If \(V\) is nonconstant on \(\mathcal{S}\), both quantities are strictly increasing for \(\gamma>0\).
\end{proposition}

\begin{proof}
Differentiating \(F\) gives
\[
F^{\prime}(\gamma)=\mathbb{E}_{\rho_\gamma}[V],
\qquad
F^{\prime\prime}(\gamma)=\operatorname{Var}_{\rho_\gamma}(V).
\]
Moreover, Eq.~(\ref{eq:app_closed_form_target}) implies
\[
D_{\mathrm{KL}}(\rho_\gamma\|p)
=\gamma F^{\prime}(\gamma)-F(\gamma).
\]
Differentiating this identity proves
Eqs.~(\ref{eq:app_value_derivative})--(\ref{eq:app_kl_derivative}).
Because \(p\) and hence \(\rho_\gamma\) have full support on
\(\mathcal{S}\), the variance is positive whenever \(V\) is nonconstant.
\end{proof}

This monotonicity makes the role of \(\gamma\) precise: increasing it yields a target with higher estimated downstream value while moving it farther from the teacher prior. It also yields a one-to-one correspondence between \(\gamma\) and the induced radius over the nondegenerate portion of the path. Specifically, let
\(\mathcal{A}^{\star}=\arg\max_{v\in\mathcal{S}}V(v)\) and
\(P^{\star}=\sum_{v\in\mathcal{A}^{\star}}p(v)\). Then
\begin{equation}
\lim_{\gamma\to0}\rho_\gamma=p,\qquad
\lim_{\gamma\to\infty}\rho_\gamma(v)
=
\begin{cases}
p(v)/P^{\star}, & v\in\mathcal{A}^{\star},\\
0, & v\notin\mathcal{A}^{\star},
\end{cases}
\label{eq:app_target_limits}
\end{equation}
and the induced radius increases from \(0\) to
\(-\log P^{\star}\). Notably, when several candidates share the maximum value, the large-\(\gamma\) limit preserves their relative teacher probabilities rather than selecting an arbitrary one.

\begin{corollary}[Diminishing returns of KL budget]
\label{cor:diminishing_returns}
Assume that \(V\) is nonconstant, and let
\[
G(\epsilon)=\max_{\rho:\,D_{\mathrm{KL}}(\rho\|p)\leq\epsilon}
\mathbb{E}_{\rho}[V]
\]
for \(0<\epsilon<-\log P^{\star}\). If \(\gamma(\epsilon)\) denotes the unique inverse temperature inducing radius \(\epsilon\), then
\begin{equation}
G^{\prime}(\epsilon)=\frac{1}{\gamma(\epsilon)},
\qquad
G^{\prime\prime}(\epsilon)
=-\frac{1}{\gamma(\epsilon)^3
\operatorname{Var}_{\rho_{\gamma(\epsilon)}}(V)}<0.
\label{eq:app_diminishing_returns}
\end{equation}
Thus, relaxing the teacher-centered KL budget improves the optimal estimated branch value, but with strictly diminishing marginal returns.
\end{corollary}

\begin{proof}
Along the calibration path,
\(G(\epsilon(\gamma))=\mathbb{E}_{\rho_\gamma}[V]\).
Proposition~\ref{prop:monotone_tradeoff} gives
\(\mathrm{d}G/\mathrm{d}\gamma=
\operatorname{Var}_{\rho_\gamma}(V)\) and
\(\mathrm{d}\epsilon/\mathrm{d}\gamma=
\gamma\operatorname{Var}_{\rho_\gamma}(V)\).
Applying the chain rule once gives
\(G^{\prime}(\epsilon)=1/\gamma\); differentiating once more yields Eq.~(\ref{eq:app_diminishing_returns}).
\end{proof}

\begin{corollary}[Finite-temperature teacher anchoring]
\label{cor:teacher_anchoring}
Let \(\Delta_V=\max_vV(v)-\min_vV(v)\). For every finite
\(\gamma>0\) and \(v\in\mathcal{S}\),
\begin{equation}
\exp(-\gamma\Delta_V)
\leq\frac{\rho_\gamma(v)}{p(v)}
\leq\exp(\gamma\Delta_V).
\label{eq:app_density_ratio_bound}
\end{equation}
Consequently, calibration preserves the teacher top-\(k_p\) support at every finite temperature.
\end{corollary}

\begin{proof}
Since
\(\exp(\gamma\min_vV(v))\leq Z(\gamma)
\leq\exp(\gamma\max_vV(v))\), the result follows directly from
\(\rho_\gamma(v)/p(v)=\exp(\gamma V(v))/Z(\gamma)\).
\end{proof}

\subsection{Estimated-Value Improvement under Teacher Anchoring}
\label{app:value_improvement}

\begin{proposition}[Improvement--deviation bounds]
\label{prop:value_improvement}
The outcome-calibrated target satisfies
\begin{equation}
\mathbb{E}_{\rho_\gamma}[V]-\mathbb{E}_{p}[V]
\geq
\frac{1}{\gamma}D_{\mathrm{KL}}(\rho_\gamma\|p)
\geq0.
\label{eq:app_value_improvement}
\end{equation}
Moreover, with
\(\epsilon(\gamma)=D_{\mathrm{KL}}(\rho_\gamma\|p)\),
\begin{equation}
0\leq
\mathbb{E}_{\rho_\gamma}[V]-\mathbb{E}_{p}[V]
\leq
\Delta_V\sqrt{\frac{\epsilon(\gamma)}{2}}.
\label{eq:app_value_upper_bound}
\end{equation}
\end{proposition}

\begin{proof}
Optimality of \(\rho_\gamma\) in
Eq.~(\ref{eq:app_regularized_target}), using the teacher prior \(p\) as a feasible comparator, gives
\[
\mathbb{E}_{\rho_\gamma}[V]
-\frac{1}{\gamma}D_{\mathrm{KL}}(\rho_\gamma\|p)
\geq\mathbb{E}_{p}[V],
\]
which proves Eq.~(\ref{eq:app_value_improvement}). For the upper bound, the expectation difference is at most
\(\Delta_V\,\mathrm{TV}(\rho_\gamma,p)\). Pinsker inequality gives
\(\mathrm{TV}(\rho_\gamma,p)\leq
\sqrt{D_{\mathrm{KL}}(\rho_\gamma\|p)/2}\), yielding
Eq.~(\ref{eq:app_value_upper_bound}).
\end{proof}

Equation~(\ref{eq:app_value_improvement}) formalizes the benefit of the exponential tilt, whereas Eq.~(\ref{eq:app_value_upper_bound}) formalizes its conservatism. The guarantee concerns the probed estimate \(V=\hat V_t\); it does not by itself assert improvement in the unknown population branch value.

\subsection{Preference Correction and the Binary-Verifier Case}
\label{app:binary_verifier}

The target changes a teacher preference only when the outcome advantage is large enough to overcome the prior log-odds. In particular, for candidates \(u,v\in\mathcal{S}\),
\begin{equation}
\rho_\gamma(v)>\rho_\gamma(u)
\quad\Longleftrightarrow\quad
\gamma\bigl(V(v)-V(u)\bigr)
>
\log\frac{p(u)}{p(v)}.
\label{eq:app_preference_reversal}
\end{equation}
Thus, \(\gamma\) sets an explicit evidence threshold for reversing a teacher ranking. Candidates with equal values preserve their teacher-relative odds, and adding the same constant to all branch values leaves the target unchanged.

The current training verifier is binary and uses one probe rollout per candidate, so the realized branch-value estimates lie in \(\{0,1\}\). This case admits a particularly direct interpretation.

\begin{proposition}[Binary-verifier odds update]
\label{prop:binary_verifier}
Let
\(\mathcal{S}^{+}=\{v\in\mathcal{S}:V(v)=1\}\), and define the prior and calibrated probability masses on successful branches as
\[
P_{+}=\sum_{v\in\mathcal{S}^{+}}p(v),
\qquad
Q_{+}=\sum_{v\in\mathcal{S}^{+}}\rho_\gamma(v).
\]
Then
\begin{equation}
Q_{+}
=\frac{e^\gamma P_{+}}{1-P_{+}+e^\gamma P_{+}},
\qquad
\frac{Q_{+}}{1-Q_{+}}
=e^\gamma\frac{P_{+}}{1-P_{+}},
\label{eq:app_binary_odds}
\end{equation}
Moreover,
\begin{equation}
D_{\mathrm{KL}}(\rho_\gamma\|p)
=
D_{\mathrm{KL}}\left(
\operatorname{Bern}(Q_{+})\,\|\,
\operatorname{Bern}(P_{+})
\right).
\label{eq:app_binary_kl}
\end{equation}
These identities hold whenever \(0<P_{+}<1\). Within either outcome group, the target preserves the teacher-relative probabilities.
\end{proposition}

\begin{proof}
For successful branches the exponential factor in
Eq.~(\ref{eq:app_closed_form_target}) is \(e^\gamma\), whereas for unsuccessful branches it is \(1\). Hence
\[
Z(\gamma)=e^\gamma P_{+}+(1-P_{+}),
\]
and summing Eq.~(\ref{eq:app_closed_form_target}) over
\(\mathcal{S}^{+}\) gives Eq.~(\ref{eq:app_binary_odds}). If two candidates have the same binary value, their exponential factors cancel in their probability ratio, leaving the corresponding teacher odds unchanged. Hence the conditional distributions within \(\mathcal{S}^{+}\) and its complement are unchanged. Applying the chain rule for KL divergence over this binary partition yields Eq.~(\ref{eq:app_binary_kl}).
\end{proof}

Equation~(\ref{eq:app_binary_kl}) shows that calibration spends its entire KL deviation on reallocating probability mass between the two outcome groups, introducing no within-group distortion. This proposition also clarifies the interaction between probing, gating, and calibration. If \(P_{+}=0\), the position is removed by the \(\mathcal{B}^{+}\) gate. If \(P_{+}=1\), all candidates have equal value and \(\rho_\gamma=p\). In the informative mixed case, outcome calibration multiplies the aggregate prior odds of successful versus unsuccessful branches by exactly \(e^\gamma\), while retaining the finer teacher preferences within each group.

\section{Implementation Details}
\label{app:more_experiment_details}
\label{app:implementation_details}

\paragraph{Off-policy Training.}
Following \citep{jin2026entropyawareonpolicydistillationlanguage}, KD is implemented with DistillKit \citep{distillkit2024}, with one teacher response sampled per training problem. Table~\ref{tab:kd_main_hparams} summarizes the off-policy configuration.

\begin{table}[H]
\centering
\caption{Hyperparameters used for off-policy distillation.}
\label{tab:kd_main_hparams}
\small
\renewcommand{\arraystretch}{1.08}
\setlength{\tabcolsep}{7pt}
\begin{tabular}{@{}lc@{}}
\toprule
\textbf{Hyperparameter} & \textbf{Value} \\
\midrule
Learning rate & $1\times10^{-5}$ \\
LR scheduler type & cosine \\
Optimizer & AdamW \\
CE loss weight & 0.5 \\
Forward-KL loss weight & 0.5 \\
Training batch size & 128 \\
Training epoch & 3 \\
Cutoff length & 4{,}096 \\
Top-$k$ (for FKL) & 16 \\
\bottomrule
\end{tabular}
\end{table}

\paragraph{On-policy Training.}
We implement OPD, EOPD, GRPO, and \alg{} with \texttt{verl} \citep{Sheng_2025} and generate rollouts asynchronously with \texttt{SGLang} \citep{zheng2024sglangefficientexecutionstructured}. All main runs use 4 $\times$ A800 GPUs, FSDP2, and bfloat16. Table~\ref{tab:onpolicy_main_hparams} summarizes the principal training hyperparameters.

\begin{table}[H]
\centering
\caption{Hyperparameters used for on-policy distillation and GRPO.}
\label{tab:onpolicy_main_hparams}
\small
\renewcommand{\arraystretch}{1.08}
\setlength{\tabcolsep}{14pt}
\begin{tabular}{@{}l|c|c@{}}
\toprule
\textbf{Hyperparameter} & \textbf{OPD, EOPD, SPOT} & \textbf{GRPO} \\
\midrule
Learning rate & $3\times10^{-6}$ & $3\times10^{-6}$ \\
LR scheduler type & cosine & cosine \\
Optimizer & AdamW & AdamW \\
Training batch size & 128 & 128 \\
Mini-batch size & 32 & 32 \\
Samples per prompt & 1 & 8 \\
Temperature & 1.0 & 1.0 \\
Top-$p$ & 1.0 (Qwen), 0.8 (Llama) & 1.0 (Qwen), 0.8 (Llama) \\
Max prompt length & 1{,}024 & 1{,}024 \\
Max response length & 4{,}096 & 4{,}096 \\
Clipping ratio & 0.2 & 0.2 \\
\bottomrule
\end{tabular}
\end{table}

All methods use AdamW, weight decay $0.01$, gradient clipping at $1.0$, and four mini-batch updates per rollout iteration; the distillation methods use no learning-rate warmup. EOPD applies a unit-weight top-16 forward-KL loss at teacher-entropy values of at least $0.8$, while GRPO uses KL and entropy coefficients of $10^{-3}$ and $0$, respectively.

\paragraph{Training Schedule.}
Table~\ref{tab:main_training_schedule} records the scale-specific data and effective schedules of the main on-policy runs.

\begin{table}[H]
\centering
\caption{Training data and schedules. Iterations denote rollout iterations.}
\label{tab:main_training_schedule}
\small
\setlength{\tabcolsep}{4pt}
\begin{tabular}{@{}llrrr@{}}
\toprule
Student & Training set & Examples & Epochs & Iterations \\
\midrule
Qwen3-0.6B-Base & MATH & 7{,}500 & 3 & 174 \\
Qwen3-1.7B-Base & MATH & 7{,}500 & 3 & 174 \\
Qwen3-4B-Base & DAPO-Math & 14{,}116 & 2 & 220 \\
\bottomrule
\end{tabular}
\end{table}

Training examples are shuffled, no validation set is used, and prompts exceeding the listed length limit are filtered.

\paragraph{Chat Template.}
Following EOPD, the Qwen runs of OPD, EOPD, and \alg{} use the Qwen3-8B teacher's non-thinking chat format. With thinking disabled, the Qwen3-Base tokenizer renders the same prompt: \texttt{<|im\_start|>user\textbackslash n\{query\}<|im\_end|>\textbackslash n\allowbreak <|im\_start|>assistant\textbackslash n<think>\textbackslash n\textbackslash n\allowbreak </think>\textbackslash n\textbackslash n}. GRPO instead uses the default Qwen3-Base generation prompt, which omits the empty thinking block: \texttt{<|im\_start|>user\textbackslash n\{query\}<|im\_end|>\textbackslash n\allowbreak <|im\_start|>assistant\textbackslash n}.

\paragraph{SPOT-specific Training.}
Table~\ref{tab:spot_main_hparams} summarizes the hyperparameters specific to \alg{}; the acquisition, probing, gating, and loss are defined in Algorithm~\ref{alg:spot} and Section~\ref{sec:spot}.

\begin{table}[H]
\centering
\caption{Hyperparameters used for SPOT.}
\label{tab:spot_main_hparams}
\small
\renewcommand{\arraystretch}{1.08}
\setlength{\tabcolsep}{6pt}
\begin{tabular}{@{}lc@{}}
\toprule
\textbf{Hyperparameter} & \textbf{Value} \\
\midrule
Selected positions $M$ & 2 \\
Scoring top-$k$ $k_s$ & 16 \\
Probing top-$k$ $k_p$ & 4 \\
$\lambda_{\mathrm{mass}},\lambda_{\mathrm{shape}}$ & $0.5,0.5$ \\
Reward-tilt coefficient $\gamma$ & 1.0 \\
Branch-loss weight $\beta$ & 0.1 \\
Probe rollouts per candidate & 1 \\
Training verifier & binary rule-based boxed-answer match \\
\bottomrule
\end{tabular}
\end{table}


\paragraph{Evaluation Details.}
\label{app:evaluation_details}
We use vLLM with bfloat16 precision and random seed 42, render one user message with the model's native chat template and thinking disabled, and append the instruction \texttt{Please reason step by step, and put your final answer within \textbackslash boxed\{\}.} We extract the final balanced \texttt{\textbackslash boxed\{\}} expression and use Math-Verify \citep{Kydlicek_Math-Verify_Math_Verification} for symbolic equivalence; missing or unparsable answers are counted as incorrect.
\begin{table}[H]
\centering
\caption{Sizes of the evaluation benchmarks.}
\label{tab:main_evaluation_protocol}
\small
\renewcommand{\arraystretch}{1.08}
\setlength{\tabcolsep}{12pt}
\begin{tabular}{@{}lr@{}}
\toprule
\textbf{Benchmark} & \textbf{Problems} \\
\midrule
MATH-500 & 500 \\
AMC 2023 & 40 \\
Minerva Math & 272 \\
HMMT 2025 (Feb.+Nov.) & 60 \\
AIME 2024 & 30 \\
AIME 2025 & 30 \\
\bottomrule
\end{tabular}
\end{table}

\section{More Ablation Experiment Results}
\label{app:more_ablation_results}

\subsection{Out-of-Domain Generalization}
\label{app:ood_generalization_results}

Table~\ref{tab:ood_results} reports the complete out-of-domain results
analyzed in Section~\ref{subsec:ood_generalization}. Together, these benchmarks
separate deliberative reasoning, broad knowledge, and instruction following,
providing a more diagnostic view of transfer than any single aggregate score.

\begin{table*}[t]
\centering
\colorlet{oodours}{blue!7}
\newcommand{\oodcell}[1]{\cellcolor{oodours}#1}
\caption{
Out-of-domain results for Qwen3-1.7B-Base trained on MATH (all values in \%).
GPQA-Diamond uses eight samples per question; MMLU-Pro uses category-matched
five-shot CoT demonstrations and reports Pass@1; and AlpacaEval 2.0 reports
win rate (WR) and length-controlled win rate (LC-WR). 
\textbf{Bold} and \underline{underlined} entries denote the best and
second-best 
results, respectively. The \alg{} column is shaded.
AlpacaEval scores marked
$^\dagger$ use 803 (EOPD) and 800 (\alg{}) valid judgments
out of 805, respectively.
}
\label{tab:ood_results}
\renewcommand{\arraystretch}{1.15}
\setlength{\tabcolsep}{7pt}
\begin{tabular}{@{}llccccc@{}}
\toprule
Benchmark & Metric & KD & GRPO & OPD & EOPD & \oodcell{\textbf{\alg{} (Ours)}} \\
\midrule
\multirow{2}{*}{GPQA-Diamond}
& Avg@8  & 21.10 & 19.26 & 25.00 & \second{27.21} & \oodcell{\best{29.42}} \\
& Pass@8 & 62.18 & 55.05 & 67.17 & \second{69.19} & \oodcell{\best{80.81}} \\
\midrule
MMLU-Pro
& Pass@1 & 38.13 & 41.86 & 41.23 & \best{42.90} & \oodcell{\second{42.26}} \\
\midrule
\multirow{2}{*}{AlpacaEval 2.0}
& LC-WR & 24.10 & 21.86 & 27.08 & \best{28.13}$^\dagger$ & \oodcell{\second{27.59}$^\dagger$} \\
& WR & 27.58 & 25.71 & 31.18 & \second{33.23}$^\dagger$ & \oodcell{\best{34.63}$^\dagger$} \\
\bottomrule
\end{tabular}
\end{table*}

The comparison reveals a structured transfer profile rather than uniform
dominance. Relative to the strongest baseline on GPQA-Diamond, \alg{} improves
Avg@8 by 2.21 points and Pass@8 by 11.62 points, and it leads AlpacaEval WR
by 1.40 points. On MMLU-Pro, \alg{} ranks second at 42.26: it trails EOPD
by only 0.64 points, while exceeding GRPO and OPD by 0.40 and 1.03 points,
respectively. This near-best result indicates that branch calibration preserves
broad subject-matter reasoning despite training only on MATH, while EOPD's
edge suggests that token-level uncertainty matching remains especially useful
for heterogeneous knowledge questions. On AlpacaEval, length control reverses
the raw-WR ordering, indicating that part of the preference gain is associated
with response length. Overall, outcome calibration transfers most clearly to
deliberative reasoning, with method-dependent tradeoffs on broader forms of
generalization.

\subsection{Consistency Across Model Families}
\label{app:model_family_results}

To evaluate 
the generality of our method
beyond Qwen3, we repeat the comparison with a
Llama-3.2-3B-Instruct student trained on MATH for three epochs and a
Llama-3.1-8B-Instruct teacher for the distillation methods. All methods retain
their native optimization protocols and are evaluated with eight samples under
identical decoding. Because this setting differs from the Qwen experiments in
architecture, checkpoint type, and student--teacher scale, Table~\ref{tab:ablation_model_family} compares methods only within Llama; the
experiment tests transfer of the relative method pattern, not absolute
performance across families.

\begin{table}[H]
\centering
\definecolor{familyavg}{HTML}{EBF3FC}
\colorlet{familypass}{green!10}
\newcommand{\familyavgcell}[1]{\cellcolor{familyavg}#1}
\newcommand{\familypasscell}[1]{\cellcolor{familypass}#1}
\caption{
\textbf{Results on the Llama model family.}
\textit{Macro avg.}~is the unweighted mean over the four benchmarks,
computed before rounding. \textbf{Bold} and \underline{underlined} entries
denote the best and second-best results, respectively. SPOT cells are
shaded in blue for Avg@8 and green for Pass@8.
}
\label{tab:ablation_model_family}
\small
\setlength{\tabcolsep}{4pt}
\renewcommand{\arraystretch}{1.15}
\begin{tabular}{@{}llcccc@{}}
\toprule
\textbf{Benchmark} & \textbf{Metric} & \textbf{GRPO} & \textbf{OPD}
& \textbf{EOPD} & \textbf{SPOT} \\
\midrule
\multirow{2}{*}{MATH500}
& Avg@8  & \second{43.65} & 35.90 & 37.75 & \familyavgcell{\best{43.78}} \\
& Pass@8 & 61.60 & 63.60 & \second{65.20} & \familypasscell{\best{66.20}} \\
\midrule
\multirow{2}{*}{AMC23}
& Avg@8  & \best{24.06} & 13.13 & 19.06 & \familyavgcell{\second{23.44}} \\
& Pass@8 & \second{47.50} & 45.00 & \second{47.50} & \familypasscell{\best{50.00}} \\
\midrule
\multirow{2}{*}{AIME24}
& Avg@8  & \second{2.08} & 0.83 & 1.67 & \familyavgcell{\best{2.50}} \\
& Pass@8 & \best{16.67} & 6.67 & \second{13.33} & \familypasscell{\best{16.67}} \\
\midrule
\multirow{2}{*}{AIME25}
& Avg@8  & \second{0.00} & \best{0.42} & \second{0.00} & \familyavgcell{\best{0.42}} \\
& Pass@8 & \second{0.00} & \best{3.33} & \second{0.00} & \familypasscell{\best{3.33}} \\
\midrule
\multirow{2}{*}{Macro avg.}
& Avg@8  & \second{17.45} & 12.57 & 14.62 & \familyavgcell{\best{17.53}} \\
& Pass@8 & 31.44 & 29.65 & \second{31.51} & \familypasscell{\best{34.05}} \\
\bottomrule
\end{tabular}
\end{table}

\paragraph{Within-family comparison.}
Within the Llama setting, \alg{} raises macro Avg@8 from 12.57 for OPD to
17.53 and macro Pass@8 from 29.65 to 34.05. Relative to the strongest baseline
for each metric, however, the distinction is asymmetric: Avg@8 is essentially
tied with GRPO (17.53 vs.\ 17.45), whereas Pass@8 exceeds EOPD by 2.54 points.
The benchmark-level results show the same pattern. \alg{} is best or tied-best
in Pass@8 on all four benchmarks, while GRPO remains stronger in AMC23 Avg@8
and OPD ties \alg{} on AIME25. Thus, the evidence supports a more consistent
improvement in multi-sample solution coverage than in average per-sample
accuracy, rather than uniform dominance on every benchmark.

\paragraph{Cross-family interpretation.}
This asymmetry mirrors the Qwen results: across the 0.6B and 1.7B students
trained on MATH and the 4B student trained on DAPO, \alg{} leads macro Pass@8
in every setting, while macro Avg@8 is best or nearly tied with the strongest
baseline. Its recurrence across Base and Instruct checkpoints, two model
families, multiple capacities, and two training-data regimes makes a
Qwen-specific explanation less plausible. The aggregates are not directly
comparable because the Llama panel contains four benchmarks whereas the Qwen
panel contains six; moreover, architecture, capacity, initialization, teacher
pairing, and training data are not independently controlled. We therefore view
these results as evidence of portability across the evaluated configurations,
not architecture invariance or a causal scaling law.
\end{document}